\documentclass[twocolumn]{autart}
\usepackage{graphicx}
\RequirePackage{amsmath, amssymb, graphicx, url, xspace}
\newtheorem{ex}{Example}
\newcommand{\R}{\mathbb{R}}

\DeclareRobustCommand*\eg{\emph{e.g.,}\@\xspace}

\DeclareRobustCommand*\ie{\emph{i.e.,}\@\xspace}
\DeclareMathOperator*\argmin{arg\,min}

\newtheorem{rmk}{Remark}[section]
\newtheorem{proof}{Proof}
  \RequirePackage{xspace} \RequirePackage{amsmath, amssymb, textcomp,color}
 
  \DeclareRobustCommand*\eg{\emph{e.g.,}\@\xspace}
  
  \DeclareRobustCommand*\ie{\emph{i.e.,}\@\xspace}
  
  \DeclareMathOperator*{\minimize}{minimize}

  \DeclareMathOperator*{\subject}{subj. to}
  
  \DeclareMathOperator*{\maximum}{max}
 \DeclareMathOperator*{\Ti}{\Omega}

 \newcommand{\x}[1]{\theta #1}
  \newcommand{\T}{\mathsf{T}}
 \newcommand{\TODO}[1]{{\color{red} \textbf{#1}}}
\newcommand{\CC}[1]{\varphi #1}

\begin{document}
\begin{frontmatter}
\title{Scalable  Anomaly Detection in Large  Homogenous Populations\thanksref{footnoteinfo}}
\thanks[footnoteinfo]{A preliminary version of this paper was presented in the 16th IFAC Symposium on System Identification, Brussels, Belgium, July 11-13, 2012.}
\author[Liu,Berk]{Henrik Ohlsson}\ead{ohlsson@isy.liu.se},
\author[Liu]{Tianshi Chen}\ead{tschen@isy.liu.se},
\author[Liu]{Sina  Khoshfetrat Pakazad}\ead{sina.kh.pa@isy.liu.se},
\author[Liu]{Lennart Ljung}\ead{ljung@isy.liu.se},
\author[Berk]{S. Shankar Sastry}\ead{sastry@eecs.berkeley.edu}

 \address[Liu]{Department of Electrical Engineering, Link\"oping University,
 SE-581 83 Link\"oping, Sweden}

 \address[Berk]{Department of Electrical Engineering and Computer
  Sciences, University of California at Berkeley, CA, USA}

\begin{abstract}
Anomaly detection in large populations is a challenging but highly
relevant problem. The problem is essentially a multi-hypothesis
problem, with a hypothesis for every division of the systems into
normal and anomal systems. The number of hypothesis grows rapidly
with the number of systems  and approximate solutions
become a necessity for any problems of practical interests.
In the current  paper we
take an optimization approach to this multi-hypothesis problem. We
first observe that the  problem is equivalent to a
non-convex combinatorial optimization problem.
We then relax the problem to a convex problem that can be solved distributively on the
systems and that stays computationally tractable as the
number of systems increase.
An interesting property of the proposed method is that it
 can  under certain conditions  be shown
 to give exactly the same result as the combinatorial multi-hypothesis
 problem  and  the relaxation is hence tight.
\end{abstract}
\end{frontmatter}

\section{Introduction}\label{sec:intro}
In this paper we study the following problem: We are given $N$
systems and we suspect that $k \ll N$ of them behave differently
from the majority. We do not know beforehand what the normal
behavior is, and we do not know which $k$ systems that behave
differently. This problem is known as an \emph{anomaly detection
problem} and has been discussed \eg in
\cite{Chandola:2009,Chu:11,Gorinevsky12}. It clearly has links to
\emph{change detection} (\eg
\cite{PattonFC:89,BassevilleN:93,Gustafsson:01}) but is different
because  the detection of anomalies is done by comparing systems
rather than looking for  changes over time.

The anomaly detection problem typically becomes very computationally demanding, and
it is  therefore of interest to study \emph{distributed solutions}. A
distributed solution is also motivated 
by that many anomaly detection
problems are spatially distributed and lack a central computational unit.

\begin{ex}[Aircraft Anomaly Detection]\label{ex:1}
In this example we consider the problem of detecting abnormally
behaving airplanes in a large homogenous fleet of
aircrafts. Homogenous here means that the normal aircrafts have
similar dynamics.  This is a very
relevant problem \cite{Chu:11,Gorinevsky12} and of highest interest for safety in
aeronautics. In fact, airplanes are constantly gathering data and
being monitored for this exact reason. In particular, so called
\textit{flight operations quality assurance} (FOQA) data are
collected by several airlines and used to improve their fleet's
safety.

As showed in \cite{Chu:11}, faults in the angle-of-attack channel
can be detected by studying the relation between the angle of
attack, the dynamic pressure, mass variation, the stabilizer
deflection angle, and the elevator deflection. The number of airplanes
in a fleet might be of the order of hundreds and data from a couple of
thousand flights might be available (200 airplanes and data from 5000
flights were used in \cite{Chu:11}).  Say that our goal is to find
the 3 airplanes among 200 airplanes that are the most likely to be anomal to narrow
 the airplanes that need manual inspection.
Then, we would have to
evaluate roughly
$1.3\times 10^6$ hypothesis (the number of unordered selections of 3
out of 200 airplanes). For each hypothesis, the likelihood for
the observed data would then be maximized with respect to the unknown
parameters and the most likely hypothesis accepted.
This is clearly a very computationally challenging problem. 
\end{ex}



Example \ref{ex:1}  considers anomaly detection in a large homogenous
population
and is the type of problem we are interested in
solving in this paper. 
The problem has previously been
approached using \textit{model based anomaly detection} methods, see
\eg   \cite{Chu:11,Gorinevsky12,Chandola:2009}.  This
class of anomaly detection methods is suitable to detect anomalies in
systems, as opposed to non-model based methods that are
more suitable for finding  anomalies in data. Model based anomaly
detection methods work under the assumption that the dynamics of
normal systems are the same, or equivalently, that the population of
systems is homogenous. The normal dynamics is modeled from system
observations
and most
papers  assume that an abnormal-free training data set is
available for the estimation, see for instance
\cite{abraham89,abraham79,fox72}. Some papers have been presented
to relax this assumption. In \eg \cite{Rousseeuw87}, the use of a
regression  technique robust to anomalies was suggested.

The detection of anomal systems is in
model based anomaly detection  done  by comparing  system observations
and model predictions and often done
 by a statistical test, see \eg
\cite{Eskin00,Desforges98}. However, in
non-model based anomaly detection,   classification based \cite{Tan:05,Duda00},
clustering based \cite{jain88}, nearest neighbor based \cite[Ch. 2]{Tan:05}, information
theoretic \cite{aring96} and spectral methods \cite{Parra96}  are also
common.  See  \cite{Chandola:2009} for a detailed review of
anomaly detection methods. 
Most interesting and
similar to the proposed method is the
more recent approach taken in
\cite{Chu:11,Gorinevsky12}. They simultaneously estimate the
regression model for the normal dynamics and perform
anomaly detection. The method of \cite{Chu:11} is discussed further in
the numerical section. There has also been  some  work on distributed
anomaly detection, \eg \cite{Zimmermann:2006,Chatzigiannakis06,Chu:11}.

The main contribution of the paper is a novel distributed, scalable
and model based
method for anomaly
detection in large homogenous populations. The method is distributed in the sense
that the computations can be distributed over the systems in the
population or a cluster of computers. It is scalable since the size of the
optimization problem solved on each system is independent of the number
of systems in the population.  This is made possible by a novel
formulation of the multi-hypothesis problem as a sparse problem. The
method  also shows superior
performance and is easier to tune than previously proposed model based
anomaly detection methods.  Last, the method does not need a training
data set and a regression model of the normal dynamics is estimated at
the same time as abnormal systems are detected. This is
particularly valuable since often neither a training data set or a
regression model for the normal dynamics are available.

 The remaining of the paper is organized as follows.
 Section~\ref{sec:probform} states the problem and shows the relation
 between anomaly detection and multi-hypothesis testing.
Section~\ref{sec:noncovex} reformulates the multi-hypothesis
 problem as a sparse optimization problem and Section~\ref{sec:cr} gives a
 convex formulation. The convex problem is solved in a distributed
 manner on the systems and this is discussed in
 Section~\ref{sec:Distributed}.
We return to Example \ref{ex:1} and compare to the method of
\cite{Chu:11} in Section~\ref{sec:num}.
 Finally, we conclude the paper in Section~\ref{sec:con}.


\section{Problem Statement and Formulation}\label{sec:probform}

Assume that the population of interest consists of $N$ systems. Think
for example of the $N$ airplanes studied in Example
\ref{ex:1}. Further assume  that
there is a linear unknown relation describing the relation between measurable
quantities of interest (angle of attack, the dynamic pressure, mass
variation, the stabilizer deflection angle, and the elevator
deflection in Example \ref{ex:1}):
\begin{equation}\label{eq:sensoreq}
  y_i(t)=\CC_i^\T (t)\x_{i,0}  + e_i(t), \qquad i=1,\cdots,N,
\end{equation}
where $t$ is the time index, $i$ indexing systems, $y_i(t)\in\R$ and $\CC_i(t) \in \R^{m}$
are the measurement and regressor vector at time $t$, respectively,
 $\x_{i,0}$ is the \emph{unknown} model parameter, and $e_i(t)
\in \R$ is the measurement noise.  For the $i$th system, $i=1,\cdots,N$, let
$\{(y_i(t),\CC_i(t))\}_{t=1}^{\Ti}$ denote the collected data
set and $\Ti$  the number of observations collected on each system.
We  assume that $e_i(t)$ is
white Gaussian distributed with mean zero and some \emph{unknown}  variance
$\sigma^2$ and moreover, independent of $e_j(t)$ for all $i\neq j$. However,
log-concave distributed noise could be handled with minor changes.



We will in the following say that the population behaves \textit{normally} and that
 none of the systems are abnormal if
$\x_{1,0} = \cdots = \x_{N,0}=\x_0$. Reversely, if any system has a model
 parameter deviating from the \textit{nominal} parameter value $\x_0$,
 we will consider that system as \textit{abnormal}.



To solve the problem we could argue like this: Suppose we have a
hypothesis about which $k$ systems are the anomalies. Then we could
estimate the nominal parameters $\theta_0$ by least squares from the
rest, and estimate individual $\theta_i$ for the $k$ anomalies. Since
we do not know which systems are the anomalies we have to do this for
all 
possible selections of $k$  systems  from a set of $N$. This gives a
total of
\begin{equation}
c(N,k)
=  {N !}/\big ({(N-r)! r!} \big )
\end{equation}
 possible hypotheses. To decide which  is the most likely hypothesis,
 we would evaluate the total misfit for all the systems, and choose
 that combination that gives the smallest total misfit. If we let $\gamma_j$
 be the set of assumed abnormal systems associated with the $j$th
 hypothesis $j = 1,\cdots,c(N,k)$, 
this would be equivalent to
 solving the
 the non-convex optimization problem
\begin{align}\nonumber
\minimize_{j=1,\dots,c(N,k)} &\sum_{s \in \gamma_j}\min_{\theta_{j,s}} \sum_{t=1,\cdots,\Ti} \| y_s(t) - \CC_s^\T
(t) \theta_{j,s}  \|^2\\ \label{eq:likelihood3} +&\min_{\theta_{j,0}}\sum_{s \notin \gamma_j,
  t=1,\cdots,\Ti} \| y_s(t) - \CC_s^\T (t) \theta_{j,0}  \|^2.
\end{align}
Since we assume that all systems have the same noise variance
$\sigma^2$, this is 
a formal hypothesis test.
If the systems may have different noise levels we would have to estimate these and include proper weighting in~\eqref{eq:likelihood3}.

The difficulty is how to solve ~\eqref{eq:likelihood3} when the
number of systems $N$ is large. As seen in Example~\ref{ex:1}, even
for rather small examples ($k=3,\,N=200$), the number of hypothesis
$c(N,k)$ becomes large and
the problem \eqref{eq:likelihood3}  computationally intractable. 
\section{Sparse Optimization Formulation}\label{sec:noncovex}

A key observation to be able to solve the anomaly detection problem
in a computationally efficient manner is the reformulation
of the muti-hypothesis problem \eqref{eq:likelihood3} as a sparse optimization problem.  To do
this, first
notice that the multi-hypothesis test \eqref{eq:likelihood3} will
find the $k$  systems whose data are most likely to not have been
generated from the same model as the remaining $N-k$ systems. Let us
say that $j^*$ was the selected hypothesis and denote the parameter
of the $i$th system by $\theta_i$, $i=1,\cdots,N$. Then
$\theta_{i_1} \neq \theta_{i_2}$ for all $i_1,i_2 \in \gamma_{j^*}$
and $\theta_{i_1} =\theta_{i_2}$ for all $i_1,i_2 \in
\{1,\cdots,N\}/ \gamma_{j^*}$. Note that $N-k$ systems will  have
identical parameters. An equivalent way of solving the
multi-hypothesis problem is therefore to maximize the likelihood
under the constraint that $N-k$ systems are identical. This can be
formulated as
\begin{equation}\label{eq:first}
\begin{aligned}
&\minimize_{\x_1,\cdots,\x_N,\x{}}\hspace{2mm} \sum_{i=1}^N \sum_{t=1}^{\Ti} \|y_i(t)-\CC_i^\T (t)\x_i \|^2\\
&s.t.\hspace{1mm} \bigg \| \begin{bmatrix} \|\x_1-\x\|_p
&\|\x_2-\x\|_p&
    \cdots& \|\x_N-\x\|_p \end{bmatrix} \bigg \|_0 \hspace{-0.1cm} = \hspace{-0.1cm} k,
\end{aligned}
\end{equation}
where $\theta$ denotes the nominal parameter, $\|\cdot\|_0$ is the zero-norm (pseudo norm) which counts the
number of non-zero elements of its argument. In this way, the $k$
systems most likely to be abnormal could now be identified as the ones
for which the estimated $\|\x_i-\x\|_p\neq 0$.
Note that this is exactly the
same problem
as \eqref{eq:likelihood3} and
the same hypothesis will be selected.

Since abnormal model parameters and the nominal model parameter $\theta_0$ are estimated from
the given data sets $\{(y_i(t),\CC_i(t))\}_{t=1}^{\Ti}$,
$i=1,\cdots,N$, that are subject to the measurement noise $e_i(t)$,
$i=1,\cdots,N$, they are random variables. Moreover, it is
well-known from \cite[p. 282]{Ljung:99} that if the given data sets
$\{(y_i(t),\CC_i(t))\}_{t=1}^{\Ti}$, $i=1,\cdots,N$, are
\emph{informative enough} \cite[Def. 8.1]{Ljung:99}, then for
each $i=1,\cdots,N$, the estimate of  $\x_i$ converges (as $\Ti\rightarrow\infty$) in
distribution to the normal distribution with mean $\theta_{i,0}$ and covariance
$M_i /\Ti$
where $M_i$ is a constant matrix which depends on $\theta_{i,0}$ and
the data batch $\{(y_i(t),\CC_i(t))\}_{t=1}^{\Ti}$. This implies
that as $\Ti\rightarrow\infty$, (\ref{eq:first}) will solve the  anomaly detection problem
exactly, \ie if there are $\tilde k\leq k$ systems that have  a
different model than the rest, those would be part of the hypothesis selected.

In the case where $\Ti$ is finite, our capability of
correctly detecting the anomaly working systems will be
dependent on the scale of the given $\Ti$ and $\sigma$, even with informative enough  data sets
$\{(y_i(t),\CC_i(t))\}_{t=1}^{\Ti}$, $i=1,\cdots,N$. This is due
to the fact that the variance of the estimate of $\x_i$, $i=1,\cdots,N$, decreases as $\Ti$ and $1/\sigma$
increase. As a result, larger $\Ti$ and smaller $\sigma$ allow us to detect smaller model parameter deviations and
hence increase our ability to detect abnormal behaviors.

\begin{rmk}
It should also be noted that if there is no overlap between observed
features, anomalies can not be detected even in the noise free case.
That is, if we let $Q$ be a $m \times m$-matrix with all zeros
except for element
%
%
 $(q,q)$,  which equals $1$, then if
\begin{equation*}
\varphi^\T _i(t_1) Q \varphi_j(t_2) =0,\; j=1,\cdots,i-1,i+1,\cdots,N,
\, \forall t_1,t_2,
\end{equation*}%
a deviation in element $q$ in
$\theta_{i}$ is not detectable.
It can be shown that this
corresponds to that the data is not informative enough. 
\end{rmk}

\section{Convex Relaxation}\label{sec:cr}

It follows from basic optimization theory, see for instance~\cite{boyd04}, that there exists a $\lambda>0$, also referred to as the regularization parameter, such that
\begin{multline}\label{eq:sec}
\minimize_{\x_1,\cdots,\x_N,\x{}}\hspace{2mm} \sum_{i=1}^N
\sum_{t=1}^{\Ti} \|y_i(t)-\CC_i^\T (t)\x_i \|^2 \\
 +\lambda \bigg \| \begin{bmatrix} \|\x_1-\x\|_p &\|\x_2-\x\|_p& \cdots& \|\x_N-\x\|_p \end{bmatrix} \bigg \|_0,
\end{multline}

gives exactly the same estimate for $\x_1,\cdots,\x_N,\x{},$  as
\eqref{eq:first}. However, both \eqref{eq:first} and \eqref{eq:sec}
are non-convex and combinatorial, hence unsolvable in practice.

What makes \eqref{eq:sec} non-convex is the second term. It has
recently become popular to approximate the zero-norm by its convex
envelope. That is, to replace the zero-norm by the one-norm. This is
in line with the reasoning behind Lasso \cite{Tibsharami:96} and
compressive sensing \cite{Candes:06,Donoho:06}. Relaxing the
zero-norm by replacing it with the one-norm leads to the following
convex optimization problem
\begin{align}
\minimize_{\x,\x_1,\cdots,\x_N} \hspace{2mm} \sum_{i=1}^N
\sum_{t=1}^{\Ti} 
\|y_i(t)-\CC_i^\T (t)\x_i\|^2
\label{eq:critfleet0} + \lambda \sum_{i=1}^N & \|\x-\x_i\|_p.
\end{align}%
In theory, under some conditions on $\CC_i(t)$, $k$ and the noise,
there exists a $\lambda^*>0$ such that the criterion
\eqref{eq:critfleet0} will  work  essentially as well as \eqref{eq:first} and
detect the anomal systems exactly. This is possible because
\eqref{eq:critfleet0} can be put into the form of group Lasso
\cite{Yuan06} and the theory from compressive sensing
\cite{Candes:06,Donoho:06} can therefore be applied to establish
when the relaxation is tight. This is reassuring and
motivates the proposed approach in front of  \eg \cite{Chu:11} since no
such guarantees can be given for the method presented in \cite{Chu:11}.

In practice, $\lambda$ should be chosen carefully because it decides
$k$, the number of anomal systems picked out, which correspond to
$k$ nonzeros among $\|\hat \theta-\hat \theta_i\|_p$,
$i=1,\cdots,N$. Here $\hat \theta,\hat \theta_i,i=1,\cdots,N$ denote
the optimal solution of (\ref{eq:critfleet0}). For $\lambda=0$, all
$\hat\theta_i$, $i=1,\cdots,N$, are in general different (for finite
$\Omega$) and all $N$ systems can be regarded as anomal. It can also
be shown that there exists a $\lambda^{\text{max}}>0$ (it has closed
form solution when $p=2$) such that all $\hat\theta_i$, $i=1,\cdots,N$, equal
the nominal estimate $\hat\theta$ and thus there are no anomal
systems picked out, if and only if
$\lambda\geq\lambda^{\text{max}}$. As $\lambda$ increases from 0 to
$\lambda^{\text{max}}$, $k$ decreases piecewise from $N$ to 0. To
tune $\lambda$, we consider:
\begin{itemize}

\item if $k$ is known, $\lambda$ can be tuned by trial and error
such that solving (\ref{eq:critfleet0}) gives exactly $k$ anomal
systems. Note that making use of $\lambda^{\text{max}}$ can save the
tuning efforts.

\item if $k$ is unknown and no prior knowledge other than the given data is available,
the tuning of $\lambda$, or equivalently the tuning of $k$, becomes
a model structure selection problem with different model complexity
in terms of the number of anomal systems. This latter problem can
then be readily tackled by classical model structure selection
techniques, such as Akaike's information criterion (AIC), Bayesian
information criterion (BIC) and cross validation. If necessary, a
model validation process, \eg \cite[p. 509]{Ljung:99}, can be
employed to further validate whether the chosen $\lambda$, or
equivalently the determined $k$ is suitable or not.
\end{itemize}
In what follows, we assume a suitable $\lambda$ has been found and
focus on how to solve (\ref{eq:critfleet0}) in a distributed way.

\begin{rmk}
The solutions from solving the problem in either \eqref{eq:first} or
\eqref{eq:sec} are indifferent to the choice of $p$. However, this
is not the case for the problem in \eqref{eq:critfleet0}. In
general, $p=1$ is a good choice if one is interested in detecting
anomalies in individual elements of the parameter. On the other
hand, $p>1$ is a better choice if one is interested in detecting
anomalis in the parameter as a whole.
\end{rmk}
\section{An ADMM Based Distributed Algorithm}\label{sec:Distributed}

Although solving (\ref{eq:critfleet0}) in a centralized way provides
a tractable solution to the anomaly detection problem, it can be \emph{prohibitively
expensive} to implement on a single system (computer)  for large
populations (large $N$). Indeed, in the centralized case, an optimization problem with $(N+1)m$
optimization variables and all the data would have to be solved. As the number of
systems $N$ and the number of data $\Omega$ increase, this will lead
to increasing computational complexity in both time and storage.
Also note that many collections of systems (populations) are naturally distributed and lack a central computational unit.

In this section, we further investigate how to solve
\eqref{eq:critfleet0} in a distributed way based on ADMM. An
advantage of our distributed algorithm is that each system
(computer) only requires to solve a sequence of optimization
problems with only $2m$ (independent of $N$) optimization variables.
It requests very low storage space and in particular, it does not
access the data collected on the other systems. Another advantage is
that it can guarantee the convergence to the centralized solution.
In practice, it can converge to modest accuracy within a few tens of
iterations that is sufficient for our use, see [Boyd, 2011].

We will follow the  procedure given  in \cite[Sect.
3.4.4]{bert:97} to derive the ADMM algorithm. First define
\begin{align}\label{eq:x} x&=\begin{bmatrix}  \x_1^\T &\x_2^\T
&\cdots &\x_N^\T & \x^\T
\end{bmatrix}^\T\in\R^{(N+1)m}.\end{align}
The optimization problem \eqref{eq:critfleet0} can then be rewritten
as \begin{align} \minimize_x \hspace{2mm} G(x).
\label{eq:critfleet2}
\end{align} Here, $G(x)$ is a convex function defined as \begin{align}
G(x) =  \sum_{i=1}^N \|Y_i-\Phi_i\x_i\|^2+ \lambda
\|\x-\x_i\|_p
\end{align}
where for $i = 1, \cdots, N$,
\begin{align}\nonumber\label{eq:compdata}
        & Y_i = \begin{bmatrix}
                            y_i(1) &
                            y_i(2) &
                            \cdots &
                            y_i(\Ti)&
                          \end{bmatrix}^\T
                       ,\\ &\Phi_i =
                                           \begin{bmatrix}
                                             \CC_i(1) &
                                             \CC_i(2) & \cdots & \CC_i(\Ti)
                                           \end{bmatrix}^\T
                                         .
\end{align}
As noted in \cite[p. 255]{bert:97}, the starting point of deriving
an ADMM algorithm for the optimization problem (\ref{eq:critfleet2})
is to put (\ref{eq:critfleet2}) in the following form
\begin{align}&\minimize_{x} G_1(x) + G_2(Ax)
\label{eq:critfleet2t}
\end{align} where $G_1:\R^{(N+1)m}\rightarrow
\R$, $G_2:\R^{q}\rightarrow \R$ are two convex functions and $A$ is
a suitably defined $q\times (N+1)m$ matrix. The identification of
$G_1(\cdot),G_2(\cdot)$ and $A$ from (\ref{eq:critfleet2}) is
crucial for a successful design of an ADMM algorithm. Here, the main
concern is two fold: first, we have to guarantee that the two
induced alternate optimization problems are separable with respect
to each system; second, we have to guarantee $A^{\T} A$ is nonsingular
so that the derived ADMM algorithm is guaranteed to converge to the
optimal solution of (\ref{eq:critfleet2}). We will get back to the
convergence of proposed algorithm  later in the section.

Having the above concern in mind, we identify
\begin{align}\label{eq:g1} G_1(x) &= 0,   \\\label{eq:g2}
  G_2(z)&=\sum_{i=1}^N \|Y_i-\Phi_i\alpha_i\|^2 + \lambda
\|\beta_i-\alpha_i\|_p,
\end{align} where \begin{align}\label{eq:z}
z&=\begin{bmatrix}  \alpha_1^\T &\alpha_2^\T &\cdots &\alpha_N^\T &
\beta_1^\T &\beta_2^\T &\cdots &\beta_N^\T
\end{bmatrix}^\T \in\R^{2Nm}.
\end{align}
From (\ref{eq:g1}) to (\ref{eq:z}), we have $G(x)=G_1(x)+G_2(Ax)$,
with
\begin{align}\label{eq:A}  A = \begin{bmatrix}
I_{Nm} & 0 & \cdots & 0\\
0 & I_m & \cdots  & I_m
\end{bmatrix}^\T\in\R^{2Nm\times(N+1)m}
\end{align} where $I_{Nm}$ and $I_m$ are used to denote $Nm$ and $m$ dimensional
identity matrix, respectively.
Now the optimization problem (\ref{eq:critfleet2t}) is equivalent to
the following one
\begin{equation}
\begin{aligned}&\minimize_{x,z}
G_1(x) + G_2(z)\\
&\hspace{4mm}\subject \hspace{6mm} Ax=z.\label{eq:critfleet2b}
\end{aligned}\end{equation}
Following \cite[p.~255]{bert:97}, we assign a Lagrange multiplier
vector $\nu\in\R^{2Nm}$ to the equality constraint $Ax=z$ and
further partition $\nu$ as \begin{align}\label{eq:nu} \nu =
\begin{bmatrix} u_1^\T & u_2^\T & \cdots & u_N^\T & w_1^\T & w_2^\T &
\cdots & w_N^\T
\end{bmatrix}^\T\in\R^{2Nm}
\end{align} where for $i=1,\cdots,N$,
$u_i\in\R^m,w_i\in\R^m$. Moreover, we consider the augmented
Lagrangian function
\begin{equation}
\begin{aligned}\label{eq:Aug} L_{\rho}(\x,x,\nu) &=G_1(x) + G_2(z) +\nu^\T(Ax-z) \\
& + (\rho/2) \| Ax-z \|_2^2.
\end{aligned}\end{equation}%
Then according to \cite[(4.79) -- (4.81)]{bert:97},  ADMM can be used
to approximate the solution of \eqref{eq:critfleet2b} as follows:
\begin{subequations}\label{eq:admmstand}\tiny
\begin{align}
& x^{(k+1)} = \argmin_{x} \left\{G_1(x) + (\nu^{(k)})^\T Ax +
(\rho/2) \|Ax-z^{(k)} \|_2^2 \right\},\\
&z^{(k+1)} = \argmin_{z} \left\{G_2(z) -
(\nu^{(k)})^\T z + (\rho/2)\|Ax^{(k+1)}-z\|_2^2 \right\},\\
&\nu^{(k+1)} = \nu^{(k)} + \rho(Ax^{(k+1)}-z^{(k+1)}),
\end{align}
\end{subequations} where $\rho$ is any positive number and the initial vectors
$\nu^{(0)}$ and $z^{(0)}$ are arbitrary. Taking into account
(\ref{eq:x}) to (\ref{eq:nu}), (\ref{eq:admmstand}) can be put into
the following specific form
\begin{subequations}\label{eq:admmstand2}\tiny
\begin{align}
& \x_i^{(k+1)} = \alpha_i^{(k)} - u_i^{(k)}/\rho,\\
& \x^{(k+1)} = (1/N) \sum_{i=1}^N \beta_i^{(k)} - w_i^{(k)}/\rho,\\
&\{\alpha_i^{(k+1)},\beta_i^{(k+1)}\} = \argmin_{\alpha_i,\beta_i}
\|Y_i-\Phi_i\alpha_i\|^2+ \lambda
\|\beta_i-\alpha_i\|_p - (u_i^{(k)})^\T\alpha_i\nonumber\\
&\qquad- (w_i^{(k)})^\T\beta_i + (\rho/2) \|\x_i^{(k+1)}-\alpha_i\|_2^2+ (\rho/2) \|\x^{(k+1)}-\beta_i\|_2^2, \\
&u_i^{(k+1)} = u_i^{(k)} + \rho(\x_i^{(k+1)}-\alpha_i^{(k+1)}),\\
&w_i^{(k+1)} = w_i^{(k)} + \rho(\x^{(k+1)}-\beta_i^{(k+1)}),
\end{align}
\end{subequations} where $i=1,\cdots,N$.
It is worth to note that computations in (\ref{eq:admmstand2}) are
separable with respect to each system. Therefore, we yield the
following ADMM based distributed algorithm to the optimization
problem \eqref{eq:critfleet0}:
\begin{alg}\label{alg:admm}
On the $i$th system, $i=1,\cdots,N$, do the following:
\begin{enumerate}
\vspace{-0.15cm}
\item Initialization: set the values of $\alpha_i^{(0)}, \beta_i^{(0)}, u_i^{(0)}, w_i^{(0)}$, $\rho$ and $k= 0$.
\item $ \x_i^{(k+1)} = \alpha_i^{(k)} - u_i^{(k)}/\rho$
\item Broadcast $\beta_i^{(k)}$, $w_i^{(k)}$ to the other systems, $j=1,\cdots,i-1,i+1,\cdots, N$.
\item $\x^{(k+1)} = (1/N) \sum_{i=1}^N \beta_i^{(k)} - w_i^{(k)}/\rho$
\item 
$\{\alpha_i^{(k+1)},\beta_i^{(k+1)}\} = \argmin_{\alpha_i,\beta_i}\big\{
\|Y_i-\Phi_i\alpha_i\|^2\\+ \lambda
\|\beta_i-\alpha_i\|_p - (u_i^{(k)})^\T\alpha_i
- (w_i^{(k)})^\T\beta_i \\+ (\rho/2) \|\x_i^{(k+1)}-\alpha_i\|_2^2+ (\rho/2) \|\x^{(k+1)}-\beta_i\|_2^2\big\}\nonumber$
\item $u_i^{(k+1)} = u_i^{(k)} + \rho(\x_i^{(k+1)}-\alpha_i^{(k+1)})$
\item $w_i^{(k+1)} = w_i^{(k)} + \rho(\x^{(k+1)}-\beta_i^{(k+1)})$
\item Set $k=k+1$ and return to step 2.
\end{enumerate}
\end{alg}


\subsection{Convergence of Algorithm \ref{alg:admm}}\label{alg:converg}

It is interesting and important to investigate if Algorithm
\ref{alg:admm} would converge to the optimal solution of the
optimization problem (\ref{eq:critfleet0}). The answer is
affirmative. We have the following theorem to guarantee the
convergence of Algorithm \ref{alg:admm}, which is a straightforward
application of \cite[Ch. 3, Prop. 4.2]{bert:97} to the optimization
problem~(\ref{eq:critfleet2t}).

\begin{thm}\label{thm:conv}
Consider  (\ref{eq:critfleet2t}). The sequences $\{ x^{(k)} \}$, $\{
z^{(k)} \}$ and $\{ \nu^{(k)} \}$ generated by the ADMM based
distributed Algorithm \ref{alg:admm} for any $\rho>0$, initial
vectors $\nu^{(0)}$ and $z^{(0)}$, converge. Moreover, the sequence
$\{ x^{(k)} \}$ converges to an optimal solution of the original
problem (\ref{eq:critfleet2t}). The sequence $\{Ax^{(k)}- z^{(k)}
\}$ converges to zero, and $\{ \nu^{(k)} \}$ converges to an optimal
solution of the dual problem of (\ref{eq:critfleet2t}).
\end{thm}

\begin{proof}
According to \cite[Ch. 3, Prop. 4.2]{bert:97}, we only need
to check the optimal solution set of problem (\ref{eq:critfleet2t})
is nonempty and $A^{\T}A$ is invertible. These two assumptions are
clearly satisfied. So the conclusion follows.
\end{proof}

\begin{rmk}
In practice, Algorithm \ref{alg:admm} should be equipped with
certain stopping criteria so that the iteration is stopped when a
solution with satisfying accuracy is obtained. Another issue with
Algorithm \ref{alg:admm} is that its converge rate may be slow.
Faster convergence rate can be achieved by updating the penalty
parameter $\rho$ at each iteration. Here, the stopping criterion in
\cite[Sec. 3.3]{boyd:11} and updating rule of $\rho$ in \cite[Sec.
3.4.1]{boyd:11} are used in Section \ref{sec:num}.

\end{rmk}

\section{Numerical Illustration}\label{sec:num}


In this section, we return to the aircraft anomaly detection problem
discussed in Example \ref{ex:1} and illustrate how Algorithm
\ref{alg:admm} can distributedly detect abnormal systems. In
particular, we consider a fleet of 200 aircrafts and have access to
data from 500 flights. These data are assumed to be related through
the linear relation (\ref{eq:sensoreq}) where
$m=4,N=200,\Omega=500$, and for the $i$th aircraft at the $t$th
flight, $y_i(t)$ represents the angle-of-attack and $\varphi_i(t)$
consists of mass variations, dynamic pressure, the stabilizer
deflection angle times dynamic pressure and the elevation deflection
angle times dynamic pressure, and $e_i(t)$ is assumed to be Gaussian
distributed with mean 0 and variance 0.83, \ie $e_i(t)\sim\mathcal
N(0,0.83)$.

Similar to \cite{Chu:11}, to test the robustness of the proposed approach, we allow some nominal variation and generate
the nominal parameter $\theta_{i,0}$ for a normal aircraft as a
random variable as $\theta_{i,0} \sim \mathcal N(\bar
\theta,\Sigma_\theta)$ with
\begin{align*}
&\bar \theta = \begin{bmatrix} 0.8 \\ -2.7 \\ -0.63 \\ 0.46
\end{bmatrix}, \quad \Sigma_\theta = \begin{bmatrix} 0.04 & 0.12 &
-0.02 & 0.02\\0.12 & 0.84 & -0.09 & 0.1\\ -0.02 & -0.09 & 0.03 & 0\\
0.02 & 0.1 & 0 & 0.05  \end{bmatrix}.
\end{align*}
To simulate the anomal aircrafts in the fleet, we generate
$\theta_{i,0}$ for an anomal aircraft as a random variable as
$\theta_{i,0} \sim \mathcal N(\tilde \theta,\Sigma_\theta)$ with
$\tilde\theta = \begin{bmatrix} 3.5 &
 -0.1 & -3 & 0.001 \end{bmatrix}^\T$. Moreover, aircrafts with tags 27, 161 and 183 were
simulated as anomal. The regressor $\varphi_i(t)$ for all aircrafts
at all flights is generated as a random variable as
$\varphi_i(t)\sim \mathcal N(\bar \phi, \Sigma_\phi)$ with
\begin{align*}
\bar \phi = \begin{bmatrix} 0.95 \\ -1.22 \\ -2.79 \\ 7.11
\end{bmatrix}, \quad \Sigma_\phi = \begin{bmatrix} 0.25 & -0.02 &
0.12 & -0.04\\-0.02 & 0.45 & 0.03 & -0.52\\ 0.12 & 0.03 & 1.05 &
-1.26\\ -0.04 & -0.52 & -1.26 & 3.89  \end{bmatrix}.
\end{align*}
It should be noted that this simulation setup was previously
motivated and discussed in \cite{Chu:11,Gorinevsky12}.

Fig. \ref{fig:FigureSON} shows the centralized solution (top  plot) and
decentralized solution (bottom plot) of
(\ref{eq:critfleet0}). The solutions are close to identical, which is
expected given the convergence result of Section~\ref{alg:converg}. 
The achieved performance of Algorithm \ref{alg:admm} was obtained
within 15 iterations and $\lambda$ was set to 150 to pick out the
three most likely airplanes to be abnormal. The aircrafts for which $\|\hat\theta_i -
\hat\theta\| \neq 0$ were  the 27th, 161st and 183rd
aircraft, which also were the true abnormal aircrafts.

Fig. \ref{fig:FigureRR} illustrates the performance of the method
presented in \cite{Chu:11} with three
different regularization parameters, namely, 10, 100 and 400, shown
from the top to bottom respectively. 
As seen in the plots,  $\|\hat\theta_i -
\hat\theta\| \neq 0$ also for normally working aircrafts leading to a
result which is  difficult
to interpret. In fact, $\|\hat\theta_i -
\hat\theta\|,\, i=1,\dots, N,$ will in general always be greater than zero, independent of the
choice of regularization parameter.  This is a well known result (see
for instance \cite[Sect. 3.4.3]{Hastie:01}) and
follows from the use of Tikhonov regularization in
\cite{Chu:11}. We use a sum-of-norms regularization (which essentially
is a $\ell_1$-regularization on norms) and solving
(\ref{eq:critfleet0})  therefore gives, as long as the regularization parameter is large enough,
 $\|\hat\theta_i - \hat\theta\| = 0$ for some $i$s.  The result of
\cite{Chu:11} could of course be thresholded (the solid red line in
Fig. \ref{fig:FigureRR} shows a suitable threshold). However, this adds an
extra tuning parameter and makes the tuning considerably more
difficult than for proposed method.







\begin{figure}
\begin{center}
\includegraphics[width=8.5cm]{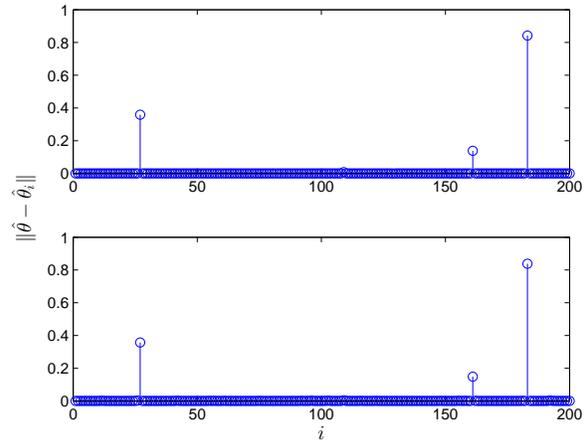}    
\vspace{-0.7cm} \caption{
The difference between the estimates of the nominal model parameter $\hat\theta$ and the
parameter $\hat\theta_i$  for the $i$th aircraft,
$i=1,\cdots,N$. The top plot shows the solution of  \eqref{eq:critfleet0} computed by a centralized
algorithm. The plot on the bottom illustrates the distributed solution computed
by Algorithm~\ref{alg:admm}.  Both the centralized and decentralized
solution were computed using a regularization parameter, $\lambda$, set to $150.$  
}
\label{fig:FigureSON}
\end{center}
\end{figure}

\begin{rmk}
A centralized algorithm for solving (\ref{eq:critfleet0}) would for
this example have to solve an optimization problem with 804
optimization variables. If  Algorithm~\ref{alg:admm} is used to
distribute the computations, 
an
optimization problem with 8 optimization variables would have to be
solved on each system (computer). This comparison
illustrates that Algorithm~\ref{alg:admm} imposes 
a much cheaper computational cost per iteration on each computer. 
In addition, the number of optimization variables is invariant to the number of systems and data.
Hence, Algorithm~\ref{alg:admm} provides a
computationally tractable and scalable solution to anomaly
detection in large homogenous populations.

\end{rmk}
\begin{figure}[t]
\begin{center}
\includegraphics[width=8.5cm]{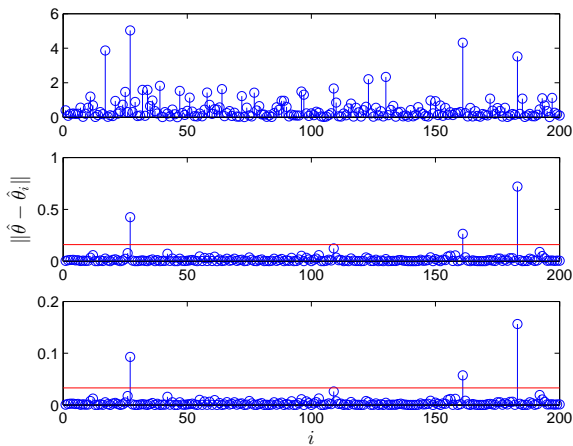}    
\vspace{-0.7cm} \caption{
The difference between the estimates of the nominal model parameter $\hat\theta$ and the
parameter $\hat\theta_i$  for the $i$th aircraft,
$i=1,\cdots,N$. The plots show the solution obtained by the method presented in
\cite{Chu:11} with regularization parameters 10, 100 and 400,
respectively from the top to bottom. 
}
\label{fig:FigureRR}
\end{center}
\end{figure}



\begin{rmk}
We used 500 flights instead of 5000 flights (5000 flights were used in
Example \ref{ex:1} and \cite{Chu:11}) to make the estimation problem more
challenging. Notice that the anomaly detection problem  still leads to
a computationally challenging
multi-hypothesis problem with roughly $1.3 \times 10^6$ hypothesis. 
\end{rmk}

\section{Conclusion}\label{sec:con}



This paper has presented a novel distributed, scalable and model based
approach to anomaly detection
for large populations. The motivation for the presented
approach is:
\begin{itemize}
\item it leads to a scalable approach to anomaly detection that can
  handle large populations,
\item it provides a purely distributed method for detecting anomaly working
systems in a collection of systems,
\item the method does not require a training data set, and
\item the algorithm is theoretically motivated by the
results derived in the field of compressive sensing.
\end{itemize}
The algorithm is based on ideas from system identification and
distributed optimization.
Basically the anomaly detection problem is first formulated
as a sparse optimization problem. This combinatorial multi-hypothesis problem can not be solved for practically interesting sizes of data
and a relaxation is therefore proposed. The convex relaxation can be
written as a group Lasso problem and theory
developed for Lasso and compressive sensing can therefore
be used to derive theoretical bounds for when the relaxation is tight.

{\small
\begin{ack} This work is partially supported by the Swedish Research
  Council in the Linnaeus center CADICS, the Swedish department of education within the ELLIIT project and the European Research Council
  under the advanced grant LEARN, contract 267381. Ohlsson is also
  supported by  a postdoctoral grant from the Sweden-America
  Foundation, donated by ASEA's Fellowship Fund, and by postdoctoral grant from the Swedish Science Foundation.
\end{ack}
}


{\small
\bibliography{refHO}}
\end{document}